\documentclass[conference,anonymous]{IEEEtran}
\IEEEoverridecommandlockouts
\usepackage{amsmath,amssymb,amsfonts}
\usepackage{algorithm}
\usepackage{algorithmic}
\usepackage{booktabs}
\usepackage{graphicx}
\usepackage{textcomp}
\usepackage{xcolor}
\usepackage[numbers,sort&compress]{natbib}
\usepackage{bm}         

\usepackage{amsthm}
\newtheorem{thm}{Theorem}
\newtheorem{defn}{Definition}

\def\BibTeX{{\rm B\kern-.05em{\sc i\kern-.025em b}\kern-.08em
    T\kern-.1667em\lower.7ex\hbox{E}\kern-.125emX}}
\begin{document}

\title{Score-based Integrated Gradient for Root Cause Explanations of Outliers}

\author{
\IEEEauthorblockN{Phuoc Nguyen, Truyen Tran, Sunil Gupta, Svetha Venkatesh}
\IEEEauthorblockA{\textit{Applied Artificial Intelligence Initiative (A2I2)} \\
\textit{Deakin University}, Australia \\
\{phuoc.nguyen,truyen.tran,sunil.gupta,svetha.venkatesh\}@deakin.edu.au}
}

\maketitle

\begin{abstract}
Identifying the root causes of outliers is a fundamental problem in causal inference and anomaly detection. Traditional approaches based on heuristics or counterfactual reasoning often struggle under uncertainty and high-dimensional dependencies. We introduce SIREN, a novel and scalable method that attributes the root causes of outliers by estimating the score functions of the data likelihood. Attribution is computed via integrated gradients that accumulate score contributions along paths from the outlier toward the normal data distribution. Our method satisfies three of the four classic Shapley value axioms—dummy, efficiency, and linearity—as well as an asymmetry axiom derived from the underlying causal structure. Unlike prior work, SIREN operates directly on the score function, enabling tractable and uncertainty-aware root cause attribution in nonlinear, high-dimensional, and heteroscedastic causal models. Extensive experiments on synthetic random graphs and real-world cloud service and supply chain datasets show that SIREN outperforms state-of-the-art baselines in both attribution accuracy and computational efficiency.
\end{abstract}


\global\long\def\d{\text{d}}%
\global\long\def\xb{\mathbf{x}}%
\global\long\def\fb{\mathbf{f}}%
\global\long\def\Xb{\mathbf{X}}%
\global\long\def\xb{\mathbf{x}}%
\global\long\def\Zb{\mathbf{Z}}%
\global\long\def\zb{\mathbf{z}}%
\global\long\def\Eb{\mathbf{E}}%
\global\long\def\thb{\bm{\theta}}%
\global\long\def\alb{\bm{\alpha}}%
\global\long\def\xib{\bm{\xi}}%
\global\long\def\epb{\bm{\epsilon}}%
\global\long\def\sib{\bm{\sigma}}%
\newcommand{\Pa}{\text{\ensuremath{\mathsf{Pa}}}}
\newcommand{\An}{\text{\ensuremath{\mathsf{An}}}} 
\newcommand{\De}{\text{\ensuremath{\mathsf{De}}}} 
\newcommand{\Ch}{\text{\ensuremath{\mathsf{Ch}}}} 

\global\long\def\E{\mathbb{E}}%
\global\long\def\R{\mathbb{R}}%
\global\long\def\KL{\text{KL}}%
\global\long\def\Dcal{{\cal D}}%
\global\long\def\Ncal{\mathcal{N}}%
\global\long\def\Mcal{{\cal M}}%
\global\long\def\Xcal{{\cal X}}%
\global\long\def\lcal{\ell}%
\global\long\def\Lcal{{\cal L}}%

\global\long\def\argmax{\operatornamewithlimits{argmax}}%
\global\long\def\argmin{\operatornamewithlimits{argmin}}%

\section{Introduction}

Identifying the root causes of outliers in complex network systems is critical across diverse domains, including manufacturing, cloud services, and distributed computing. Effective root cause analysis facilitates timely interventions, restores normal operations, and prevents costly failures. For example, in modern manufacturing, even brief system downtime can incur substantial economic losses \cite{djurdjanovic2003watchdog}. However, pinpointing the root causes of anomalies is particularly challenging in high-dimensional and noisy monitoring data, which is often exacerbated by random fluctuations and system delays \cite{ni2017ranking,pool2020lumos,yi2021semi,dhaou2021causal}.

Recent causal inference approaches leverage underlying dependency structures among variables to systematically attribute anomalies to their upstream causes \cite{lin2018microscope,liu2021microhecl,budhathoki2022causal,nguyen2024root,okati2024root}. Central to these approaches is a principled measure of "outlierness," which can be effectively decomposed according to the causal graph. For instance, \cite{budhathoki2022causal} frames outlier attribution within an information-theoretic context, employing Shapley values to quantify contributions by averaging marginal effects across variable subsets. Similarly, \cite{mahmood2020multiscale} utilizes the norm of the score function (gradient of the log density) as an information-based measure of anomaly severity.

Despite their successes, existing methods face significant practical limitations. Shapley-based approaches are computationally intensive due to combinatorial subset evaluations and may produce inconsistent, off-manifold samples. Additionally, these methods require access to original training data as empirical distributions for attribution methods such as Shapley values, raising concerns regarding data storage and privacy \cite{frye2020asymmetric,heskes2020causal}. Integrated gradient methods, such as those proposed by \cite{nguyen2024root}, offer computational advantages but are limited to linear Gaussian models.

To overcome these limitations, we propose \textbf{SIREN} (short hand for Score-based Integrated Gradient Root Cause Explanation), a novel framework that integrates conditional score modeling with diffusion-based integrated gradients. SIREN estimates the score function of data distributions conditioned on a causal structure with additive noise. We decompose the score function into a deterministic mean component and a noise-conditioned component. This allows for efficient and principled attribution of anomalies without reliance on training data at inference time. SIREN extends existing approaches by accommodating both linear and nonlinear functional causal models with heteroscedastic noise sources. Our framework reinterprets outliers as violations of causal invariance rather than merely extreme noise realizations, offering deeper insights into context-sensitive changes in system mechanisms.

Our framework relies on standard causal assumptions: data generated by a Functional Causal Model (FCM) structured as a directed acyclic graph (DAG) under causal sufficiency and independent exogenous noise terms. We assume the noise score functions are locally estimable, enabling efficient gradient-based inference along reverse-time diffusion trajectories.

\textbf{Our primary contributions include:}
\begin{enumerate}
\item A conditional score estimation approach leveraging causal structures to sample gradient-based attribution paths efficiently and without access to training data.
\item A computationally efficient decomposition of the score function into mean and noise-conditioned components, facilitating scalable attribution.
\item Establishing a theoretical link between our integrated gradient method and Shapley value theory, explicitly relating attribution scores to differences in tail probabilities.
\item Applicability to nonlinear additive noise models with arbitrary differentiable noise distributions, significantly broadening the range of practical use cases.
\end{enumerate}

\section{Preliminaries}\label{sec:Preliminaries}

We summarize the background and notations needed for our study in
this section. We use capital letters for random variables and lowercase
letters for observations.

\subsection{Outlier Measures}

Outlier measures are crucial in identifying the root causes of outliers.
From the information theory point of view, \cite{budhathoki2022causal}
defines an outlier measure of observing an event $X=x$ as: 
\begin{align}
S_{X}(x) & =-\log P\left\{ -\log p_{X}(X)\ge-\log p_{X}(x)\right\}, \label{eq:IT-score}
\end{align}
where $X$ has density $p_{X}(X)$, and $P_{}(Y)$ is the distribution
of the ``information'' $Y=-\log p_{X}(X)$. By using the information
space of $Y$ instead of the input space of $X$, this scoring function
can characterize outliers in multivariate data and and multimodal
distributions. In addition, it has the property that high measures
corresponding to low-density regions between clusters of data, which
only contain rare events significantly less represented compared to
the majority of data points \cite{budhathoki2022causal}.

\subsection{Functional Causal Models (FCMs)}

We assume there exists a true data generating process governed by
a causal structure represented by a directed acyclic graph $G$. The
nodes of the graph represent the variables generating the real world
observations, and the directed edges from the parent nodes to a node
indicate the causal dependence of that node on its parents. This causal
relationship can be captured in the following functional causal mechanism
\cite{pearl2009causality,peters2017elements}: 

The causal mechanism for each variable is represented using a functional causal model (FCM), which takes the general form:
\begin{equation}
X_j = f_j({\Pa}_j, Z_j),
\label{eq:fcm_general}
\end{equation}
where $f_j$ is a deterministic function capturing the systematic influence of the parents on $X_j$, and $Z_j \sim P(Z_j)$ is an independent noise variable specific to node $j$. 

\subsubsection{Additive Noise Model (ANM)~\cite{hoyer2009nonlinear, buhlmann2014cam}}
In this model, the noise source $Z_j$ enters Eq.~\ref{eq:fcm_general} additively:
\begin{equation}
X_j = f_j(\Pa_j) + Z_j,
\label{eq:cam}
\end{equation}
where the noise $Z_j$ is a general homoscedastic, zero-mean, and independent of the inputs. 
Together, the mean function $f_{j}$ and the noise $Z_{j}$
define the causal mechanism for node $j$. In practice, the functional
form of the $\{f_{j}\}$ can be determined using domain knowledge
or selected from an approximation family, such as linear regression
models \citep{li2022causal,budhathoki2022causal} or Bayesian linear
regression models \citep{nguyen2024root}. 
These models can then be fitted to observational data \( D = \{\mathbf{x}^{(i)}\}_{i=1}^{m} \). The noise is assumed to follow a zero-mean distribution—commonly Gaussian, but potentially uniform or any other zero-mean form. Model parameters can be learned by fitting the residuals \( X_j^{(i)} - f_j(\text{Pa}_j^{(i)}) \), assuming
\[
X_j^{(i)} \sim \mathcal{N}(f_j(\text{Pa}_j^{(i)}), \sigma_j^2),
\]
where \( \sigma_j^2 \) is constant.

\subsubsection{Location-scale Noise Model (LSN)~\cite{huang2015generalized}}
In practice, noise may not enter purely additively or with fixed variance. A more general and realistic class of models is the location-scale noise model~\cite{huang2015generalized}, in which both the mean and variance of the noise can vary with the parents:
\begin{equation}
X_j = f_j(\text{Pa}_j) + \sigma_j(\text{Pa}_j) \cdot Z_j
\label{eq:lsn}
\end{equation}

This formulation generalizes ANMs by incorporating heteroscedasticity---i.e., noise variance that depends on the input. This flexibility enables the model to capture real-world phenomena such as supply chain delay and financial volatility. These models can be fitted to observational data \( D = \{\mathbf{x}^{(i)}\}_{i=1}^{m} \) by maximizing the likelihood of the Gaussian model
\[
X_j^{(i)} \sim \mathcal{N}(f_j(\text{Pa}_j^{(i)}), \sigma_j^2(\text{Pa}_j^{(i)})),
\]
where \( f_j \) and \( \sigma_j^2 \) are parameterized functions learned from data. Together, the mean function $f_{j}$, the scale function $\sigma_{j}$, and the noise $Z_{j}$ define the causal mechanism for node $j$ with heteroscedasticity noises.
 
\subsection{Root Cause Attributions}

Based on a given causal structure, \cite{budhathoki2022causal} proposed
a method to identify root causes by attributing the outlier measure
observed at a target leaf node to the independent noise source of
every node. The authors first fit the FCMs to the normal observational
data, then inverse the independent noises $z_{j}$ from this observational
data by computing the residuals 
\begin{equation}
z_{j}=x_{j}-f_{j}\left(\mathbf{x}_{\Pa_{j}}\right)
\label{eq:residual}
\end{equation}

In order to use the Shapley attribution method \cite{shapley1953value,sundararajan2020many}
to explain an outlier event at a leaf $x_{n}=f_{n}(\Pa_{n},z_{n})$,
the authors expand the FCMs to express $x_{n}$ as a function depending
directly on all these noise variables\footnote{where we abuse notation and denote the same $f$ to mean the same
data generating process but rewritten to take input $\mathbf{z}$}: 
\begin{equation}
x_{n}=f(z_{1},\dots,z_{n})=f(\mathbf{z})\label{eq:noise-dependent-func}
\end{equation}
In this way, the noise variables $\mathbf{z}$ play the role of selecting
deterministic mechanisms. 

We clarify the sentence by grounding it in the canonical representation of Functional Causal Models (FCMs), as in Peters et al.\ (2017, Ch.\ 3). In this view, the structural equation
\[
Y := f(X, Z)
\]
can equivalently be seen as a stochastic function over $X$, where the exogenous noise $Z$ selects a deterministic function. More precisely, if we fix $Z = z$, then $Y$ becomes a deterministic function
\[
Y := f_z(X) = f(X, z),
\]
so the noise $Z$ selects which deterministic response function governs the relationship between $X$ and $Y$. Thus, $Z$ induces a distribution $P(Z)$ over the function space $Y^X$, turning the FCM into $Y := Z(X)$, where $Z \in Y^X$ is a random function. In this response function representation, the noise no longer perturbs the output directly, but selects which mechanism (function) maps inputs to outputs.

In this formulation, each draw of the noise variable $z$ corresponds to selecting a deterministic mechanism $f_z$, turning the stochastic FCM into a distribution over deterministic functions from inputs to outputs.

Thus, an outlier mechanism $j$ will be determined
by an outlier noise $z_{j}$ and is likely to be the root cause of
an observed outlier $x_{n}$ at leaf $n$. Note that the outlier measure
in Eq.\ \ref{eq:IT-score} for the leaf $n$ will depend on the noise
$\mathbf{z}$ as follows: 
\begin{align}
S_{X_{n}}(x_{n}) & =-\log P\left\{ -\log p(X_{n})\ge-\log p(x_{n})\right\} \nonumber \\
 & =-\log P\left\{ -\log p(f(\mathbf{Z}))\ge-\log p(f(\mathbf{z}))\right\} \label{eq:IT-score-f(z)}
\end{align}
where the residual $\mathbf{z}$ is calculated using Eq.\ \ref{eq:residual}.
To determine the real contribution of each node $j$ to the leaf node
outlier measure $S_{X_{n}}(x_{n})$, its sensitivity w.r.t. to an
ancestor node $j$ need to be taken into account given the context
of the remaining noise sources $\mathbf{z}_{N\backslash j}$. For
this purpose, the authors use Shapley values \cite{shapley1953value,sundararajan2020many}
a concept from cooperative game theory to correctly measure the contribution
of each noise term to $S_{X_{n}}(x_{n})$.

\paragraph{Shapley Values}

Here, the $n$ noises $\mathbf{z}=\{z_{1},\dots,z_{n}\}$ are assumed
to play a game by propagating their values through the FCMs and eventually
contributing to the outlier measure $S_{X_{n}}(x_{n})$. Let $v:2^{n}\rightarrow\mathbb{R}$
be a value function \cite{von2007theory} which map a subset $I\subseteq N$
of noise variables to a real number representing the final outlier
measure caused by this coalition $I$. The value function is defined
as \cite{lundberg2017unified}: 
\begin{equation}
v(I)=\E_{p_{Z}(z')}\left[S_{X_{n}}(z_{I}\cup z'_{\bar{I}})\right]\label{eq:value-function}
\end{equation}
where $\bar{I}=N\setminus I$ is the out-of-coalition set whose noises,
$z'_{\bar{I}}$, are marginalized out, and the outlier measure $S_{X_{n}}(\zb)$
at leaf $n$ is defined as in Eq.\ \ref{eq:IT-score-f(z)}. The target
value $v(N)$ where all observed noises are considered can then be
uniquely decomposed as 
\begin{equation}
v(N)=v(\{\})+\sum_{j\in N}\phi_{v}(j)\label{eq:decompose-shapley-value}
\end{equation}
where $v(\{\})$ is the baseline value of empty set when no noise
is considered, i.e. marginalized out, and $\phi_{v}(j)$ is the Shapley
value of each node $j$. This Shapley value is defined as\cite{sundararajan2020many}:
\begin{align}
\phi_{v}(j) & =\sum_{I\subseteq N\backslash j}\frac{|I|!(n-|I|-1)}{n!}\phi_{v}(j|I),\label{eq:shapley-marginal}\\
\text{and }\phi_{v}(j|I) & =v(I\cup j)-v(I)\label{eq:shapley-conditional}
\end{align}
where $\phi_{v}(j)$ is the marginal Shapley value by averaging over
all possible coalitions, and $\phi_{v}(j|I)$ is the Shapley value
of node $j$ given the coalition $I$. In practice, it is typical
to use subset sampling of the orderings to reduce computation. Eq.\ \ref{eq:shapley-marginal}
shows that the Shapley value $\phi_{v}(j)$ represents the marginal
contribution of the node $j$.

\paragraph{Integrated Gradient (IG)}

Instead of using Shapley values, \cite{nguyen2024root} proposed to
use integrated gradient (IG) \cite{lundberg2017unified} to compute
the outlier measure attributions to each noise variable $z_{j}$.
This method requires a noise reference which can be sampled from the
normal noise distribution $\mathbf{z}'\sim p_{\mathbf{Z}}$ to explain
the root cause of an outlier observation $x_{n}=f(\mathbf{z})$. The
attribution value of node $i$ is defined as line integral: 
\begin{align}
\text{IG}{}_{j}(\mathbf{z},\mathbf{z}') & =(z_{j}-z_{j}')\int_{t=0}^{1}\frac{\partial S_{X_{n}}(\mathbf{z}(t))}{\partial z_{j}}dt\label{eq:IG-straight-line}
\end{align}
where the line $\mathbf{z}(t)=\mathbf{z}'+t(\mathbf{z}-\mathbf{z}')$,
and $S_{X_{n}}$ is the outlier measure function in Eq.\ \ref{eq:IT-score}.
This attribution function computes the contribution of the noise $z_{i}$
to the final outlier measure by integrating the gradient of the outlier
measure along the line from $z_{i}$ to the reference $z'_{i}$, hence
the name IG. In practice, the continuous path is discretized for Riemann
integration.

\section{Methods }\label{sec:Methods}

\subsection{Score-based Integrated Gradient}

\label{subsec:Score-based-Integrated-Gradient}

Instead of computing the gradient term of the integrated gradient
(IG) attribution in Eq.\ \ref{eq:IG-straight-line}, we propose using
a score function estimator $s_{X_{n}}(\mathbf{z};\theta)$ with parameters
$\theta$ to directly output the gradient of an outlier measure at
node $X_{n}$ with respect to its input noise vector $\mathbf{z}$.
This approach is more efficient and does not require access to the
density function $p_{Z}$, which may be unknown or intractable. First,
let us define a new outlier attribution measure that relies on this
score function estimator, after which we will describe the motivation
behind these choices. 

\begin{defn}
\textbf{Score-based Outlier Measure}. Assume the causal structure
as in Section \ref{sec:Preliminaries}, and $x_{n}=f(\mathbf{z})$
as in Eq.\ \ref{eq:noise-dependent-func}. The score-based outlier
measure at leaf $X_{n}$, denoted as $S_{X_{n}}(\mathbf{z})$, is
defined as the negative log of the density of $X_{n}$: 
\begin{equation}
S_{X_{n}}(\mathbf{z})=-\log p_{X_{n}}(x_{n})=-\log p_{X_{n}}(f(\mathbf{z}))\label{eq:score-outlier-measure}
\end{equation}
\end{defn}

In order to compute the gradient of this outlier measure for outlier
attribution, we use a score function estimator $s^{n}(\mathbf{z};\theta)$
(with a lowercase ``$s$'') with parameters $\theta$ to predict
the gradient of the log density at node $X_{n}$, defined as $s_{}^{n}(\mathbf{z};\theta)=\nabla_{\mathbf{z}}\log p_{X_{n}}(f(\mathbf{z}))$.
We will describe how to train this model later. 

\begin{defn} \textbf{Baseline Outlier Attribution}\label{def:Baseline-Outlier-Attribution.}. Let
$\mathbf{x}^ {}$ be a target outlier with an observed outlier score
$S_{X_{n}}(\mathbf{z})$ at leaf $X_{n}$, and let $\mathbf{z}$ be
its inverse noise (Eq.\ \ref{eq:residual}). Let $\mathbf{z}'\sim p_{\mathbf{Z}}$
be a baseline reference, and let $\mathbf{z}(t)|_{t\in[0,1]}$ a path
from $\mathbf{z}(0)=\mathbf{z}'$ to $\mathbf{z}(1)=\mathbf{z}$.
We define the baseline outlier contribution of a node $j$ to the
target outlier score $S_{X_{n}}(\mathbf{z})$ based on the path integral
as:
\begin{align}
 & \approx\frac{1}{2}(z_{j}-z'_{j})\int_{z_{j}(0)}^{z_{j}(1)}-s_{j}^{n}(\mathbf{z}(t);\theta)_{}dz_{j}(t)\label{eq:path-measure-score}
\end{align}
where $s_{j}^{n}(\mathbf{z}(t);\theta)_{}=\nabla_{z_{j}}\log p(\mathbf{z})$
is the $j$-th component of the score estimator at node $X_{n}$.
\end{defn}

Above equations are similar to Eq.~\ref{eq:path-measure-score} in the prelim section. We reparameterize the integration variable from $dt$ to $dz$ to facilitate the gradient path integration with respect to the $z$ variable directly and scale it by $1/2$ to emphasize the triangle approximation of the tail probability area. We added this to the revised version.

The Equation (1) and (5) directly measure the tail probability as outlier measure and use it for explaining outlier. This has several drawbacks such as it is intractable due to unknown density or if using empirical estimate it is inconvenient due to data privacy and storage. We showed that by using Equation (11) to measure the log likelihood in combination with integrated gradient of the diffusion paths we can similarly approximate the tail probability as shown in Figure~(1) and thus explaining the outlier.

For implementation, we approximate the integration in Eq.\ \ref{eq:path-measure-score}
by discretizing the path $\mathbf{z}(t)$ into $k$ time steps as
follows: 
\begin{align}
\bm{\xi}_{j}(\mathbf{z},\mathbf{z}') & \approx(z_{j}-z'_{j})\sum_{i=1}^{k}-s_{j}^{n}(\mathbf{z}(t_{k});\theta)(z_{j}(t_{i})-z_{j}(t_{i-1}))\label{eq:path-discretized}
\end{align}

Next, we define the reference-free version of this definition by marginalizing
over the baseline variable. \begin{defn} \textbf{Score-based Outlier
Attribution}\label{def:Score-based-Outlier-Attribution.}. The marginal
outlier contribution of a node $j$ to the target outlier score $S_{X_{n}}(\mathbf{z})$
is defined as: 
\begin{align}
\bm{\xi}_{j}(\mathbf{z}) & =\E_{\mathbf{z}'\sim p_{\mathbf{Z}}}\bm{\xi}_{j}(\mathbf{z},\mathbf{z}')\label{eq:score-measure}\\
 & \approx\frac{1}{m}\sum_{\mathbf{z}^{(i)}\sim p_{\mathbf{Z}}|{}_{i\in\{1,\dots,m\}}}\bm{\xi}_{j}(\mathbf{z},\mathbf{z}^{(i)})\label{eq:score-measure-MC}
\end{align}
where we used $m$ Monte Carlo samples of the reference noise distribution
in Eq.\ \ref{eq:score-measure-MC} to approximate Eq.\ \ref{eq:score-measure}.
This score-based attribution definition satisfies the following axioms
for our outlier attribution problem \cite{shapley1953value,sundararajan2020many,frye2020asymmetric}:
\end{defn}
\begin{itemize}
\item \textbf{Axiom 1 (Efficiency)}: $\sum_{j\in N}\phi_{v}(j)=v(N)-v(\{\})$ 
\item \textbf{Axiom 2 (Linearity)}: \textit{$\phi_{\alpha u+\beta v}=\alpha\phi_{u}+\beta\phi_{v}$
for any value functions $u$, $v$ and any $\alpha,\beta\in\mathbb{R}$
} 
\item \textbf{Axiom 3 (Dummy)}: \textit{$\phi_{v}(j)=0$ whenever $v(I\cup j)=v(I)$
for all $I\subseteq N\backslash j$} 
\item \textbf{Axiom 4 (Asymmetry)}:\textit{ $\phi_{v}(i)=v(\{j:\pi(j)\le\pi(i)\})-v(\{j:\pi(j)<\pi(i)\})$
where $\pi$ is a topological order of the given causal graph.} 
\end{itemize}
\begin{thm} \textbf{Properties of the score-based outlier attribution}.
The score-based outlier attribution in Definition \ref{def:Score-based-Outlier-Attribution.}
satisfies the efficiency, linearity, dummy, and asymmetry axioms for
the outlier attribution problem. \end{thm}

\begin{proof} \textbf{Axiom 1 (Efficiency)}. If we absorb the constant
$\frac{1}{2}(z_{j}-z'_{j})$ into the outlier measure or the score
function, Eq.\ \ref{eq:score-measure} becomes: 
\begin{align*}
\xi_{j}(\mathbf{z},\mathbf{z}') & =\int_{z_{j}(0)}^{z_{j}(1)}\frac{\partial S_{X_{n}}(\mathbf{z}(t))}{\partial z_{j}}dz_{j}(t)
\end{align*}
Then, 
\begin{align*}
\sum_{j\in N}\xi_{j}(\mathbf{z},\mathbf{z}') & =\sum_{j\in N}\int_{z_{j}(0)}^{z_{j}(1)}\frac{\partial S_{X_{n}}(\mathbf{z}(t))}{\partial z_{j}}dz_{j}(t)\\
 & =\int_{\mathbf{z}'}^{\mathbf{z}}\frac{\partial S_{X_{n}}(\mathbf{z}(t))}{\partial\mathbf{z}}\cdot d\mathbf{z}(t)\\
 & =S_{X_{n}}(\mathbf{z})-S_{X_{n}}(\mathbf{z}')\\
 & =v(N)-v(\{\})
\end{align*}
where $v(N)$ corresponds to the outlier measure $S_{X_{n}}(\mathbf{z})$
when \textit{all} components in the outlier noise $\mathbf{z}$ is
considered, and $v(\{\})$ corresponds to \textit{none} of the components
in the outlier noise $\mathbf{z}$ being considered, i.e., exchanged
for the components in the baseline noise $\mathbf{z}'$. Thus, our
attribution method also satisfies the efficiency axiom. The baseline
reference-free version of the attribution is achieved by taking the
expectation over $\mathbf{z}'$.

\textbf{Axiom 2 (Linearity)}. The Linearity axiom follows from the
linearity of the expectation, integration, and derivation operators.
Assuming the same setting as in Definition 2, let $u$ and $v$ be
two outlier measures as defined in Definition\ 1, with corresponding
score functions $s^{n}(z(t);\theta_{u})$ and $s^{n}(z(t);\theta_{v})$,
where $\theta_{u}$ and $\theta_{v}$ are their respectively parameters,
and $\alpha,\beta\in\mathbb{R}$. We will show that $\xi(z,z';\alpha u+\beta v)=\alpha\xi(z,z';u)+\beta\xi(z,z';v)$,
where the dependence of the attribution $\xi$ on the outlier measure
is explicitly shown in its argument. For each node $j$, we have:
\begin{align*}
 & \xi_{j}(\mathbf{z},\mathbf{z}';\alpha u+\beta v)=\\
 & =\frac{1}{2}(z_{j}-z'_{j})\int_{z_{j}(0)}^{z_{j}(1)}\frac{\partial(\alpha u+\beta v)(\mathbf{z}(t))}{\partial z_{j}}dz_{j}(t)\\
 & =\frac{1}{2}(z_{j}-z'_{j})\int_{z_{j}(0)}^{z_{j}(1)}\left(\alpha\frac{\partial u(\mathbf{z}(t))}{\partial z_{j}}+\beta\frac{\partial v(\mathbf{z}(t))}{\partial z_{j}}\right)dz_{j}(t)\\
 & =\frac{1}{2}(z_{j}-z'_{j})\int_{z_{j}(0)}^{z_{j}(1)}\left(\alpha\frac{\partial u(\mathbf{z}(t))}{\partial z_{j}}+\beta\frac{\partial v(\mathbf{z}(t))}{\partial z_{j}}\right)dz_{j}(t)\\
 & =\frac{1}{2}\alpha(z_{j}-z'_{j})\int_{z_{j}(0)}^{z_{j}(1)}\frac{\partial u(\mathbf{z}(t))}{\partial z_{j}}dz_{j}(t)\\
 & +\frac{1}{2}\beta(z_{j}-z'_{j})\int_{z_{j}(0)}^{z_{j}(1)}\frac{\partial v(\mathbf{z}(t))}{\partial z_{j}}dz_{j}(t)\\
 & =\alpha\xi_{j}(\mathbf{z},\mathbf{z}';u)+\beta\xi_{j}(\mathbf{z},\mathbf{z}';v)
\end{align*}

This holds of all node $j\in N$. Therefore, the outlier attribution
method is linear with respect to the outlier measure. The baseline
reference-free version of the attribution is achieved by taking expectation
over $\mathbf{z}'$.

\textbf{Axiom 3 (Dummy)}. The Dummy axiom follows from the zero derivative
of constant functions. Suppose $v(I\cup j)=v(I)$ for all $I\subseteq N\backslash j$.
This means that the outlier measure $v$ is independent of the input
dimension $j$. Therefore, the partial derivative $\partial v(\mathbf{z}(t))/\partial z_{j}=0$
for any $\mathbf{z}(t)$. As a result, $\phi(j;v)=0$.

\textbf{Axiom 4 (Asymmetry)}. The discrepancy arises from the different
sensitivities of the leaf with respect to these nodes, as shown in
the chain rules in Eq.\ \ref{eq:chain-rule-1}, \ref{eq:chain-rule-2},
and \ref{eq:chain-rule-3}. Thus, similar noise values or changes
may be suppressed or magnified when they reach the leaf node. A similar
point has also been raised in explaining neural network predictions
\cite{frye2020asymmetric}. \end{proof} 
\begin{figure}
\begin{centering}
\includegraphics[width=0.99\columnwidth]{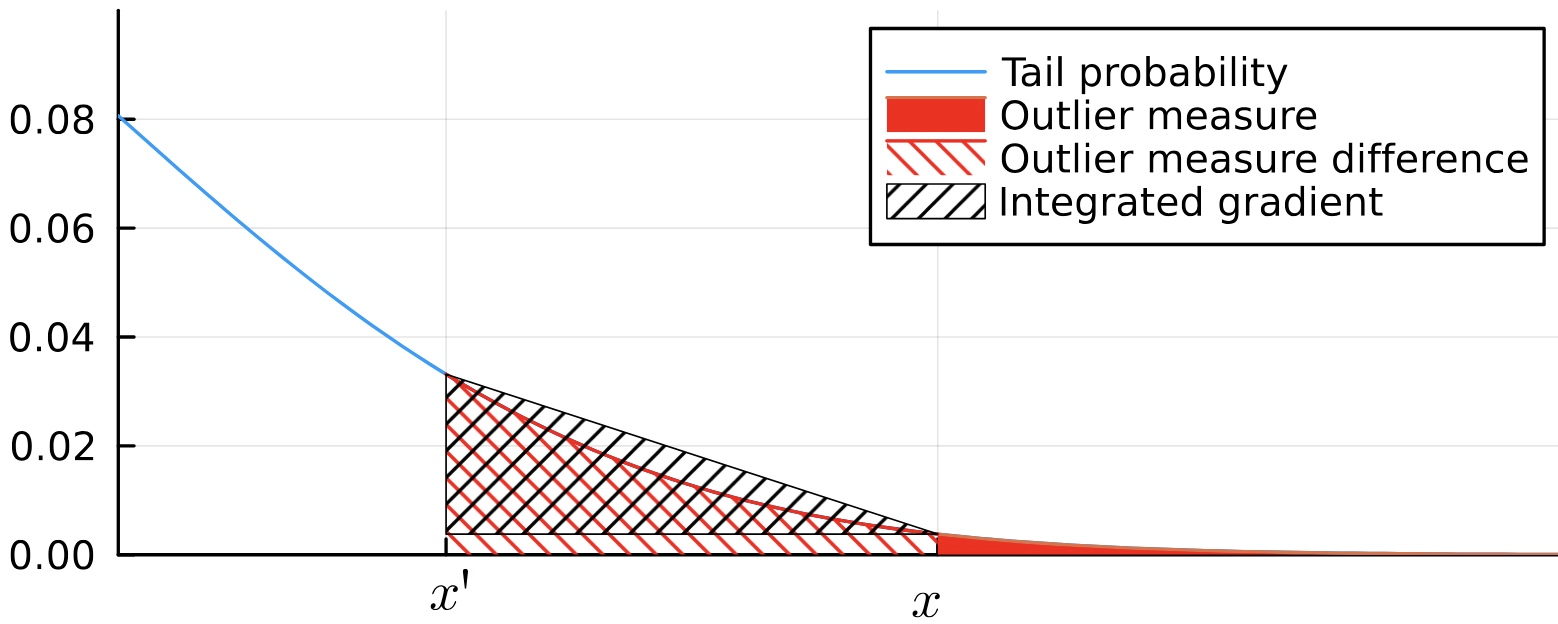}
\par\end{centering}
\caption{Outlier measure difference between a target outlier $x$ and a reference
point $x'$ as the difference in tail probabilities. For score-based
outlier attribution, the expected contribution is the triangle area
which is approximately the difference in tail probability $p(X\ge x')-p(X\ge x)$.}
\label{fig:contribution-Shapley-vs-IG}
\end{figure}

We now provide a non-rigorous motivation behind these definitions
by characterizing the expected change of the tail probability area
when moving from an outlier point $x$ to an in-distribution point
$x'$. Let us assume a simple setting where $X=X_{n}=Z_{j}$ for all
$j$. Let $p_{X_{}}$ be a continuously differentiable density function.
Then, by using the fundamental theorem of calculus, we can express
the expectation of the path integral as the difference in density
multiplied by the path distance: 
\begin{align}
\E_{p_{X_{t}}} & \frac{1}{2}(x-x')\int_{x'}^{x}\frac{\partial S_{X_{}}(x_{t})}{\partial x_{t}}dx_{t}\nonumber \\
 & =\frac{1}{2}(x-x')\int_{x'}^{x}p_{X_{t}}(x_{t})\frac{-\partial\log p_{X_{t}}(x_{t})}{\partial x_{t}}dx_{t}\nonumber \\
 & =\frac{1}{2}(x-x')\int_{x'}^{x}\frac{-\partial p_{X_{t}}(x_{t})}{\partial x}dx_{t}\nonumber \\
 & =\frac{1}{2}(x-x')(p_{X_{t}}(x')-p_{X_{t}}(x))\label{eq:triangle-area}
\end{align}

Eq.\ \ref{eq:triangle-area} represents the area of the triangle
with base $x-x'$ and height $p_{X_{t}}(x')-p_{X_{t}}(x)$. Fig.\ \ref{fig:contribution-Shapley-vs-IG}
illustrates how this triangle area approximates the difference in
tail probabilities between the outlier $x$ and the normal $x'$.
Since tail probability indicates the extremeness of $x$, the increase
in tail probability from the outlier x to a normal x' represents the
contribution of $x$ to the outlier measure when $x$ is to be exchanged
for a normal $x'$. This explanation aligns with the Shapley values
in Eq.\ \ref{eq:shapley-conditional} where the contribution of a
noise $z_{j}$, denoted as $\phi_{v}(j|I)$, is the difference between
the outlier measure when $z_{j}$ is included, $v(I\cup j)$, and
when it is excluded, $v(I)$ - that is, when it is exchanged for a
normal $z'_{j}$.

\subsection{Score Function Estimator}

Now, we describe the method to estimate the score function $s^{n}(\mathbf{z};\theta)$
with parameters $\theta$ to predict the gradient of the log density
for a node $n$ w.r.t. the noise vector $\mathbf{z}$ in Eq. \ref{eq:path-measure-score}.
Here, we will estimate a set of score functions $\{s^{i}(\mathbf{z};\theta)\}_{i=1,\dots,n}$
for every node $i$ in the graph $G$, given its ancestor noise vector
$\mathbf{z}$. Next, we will analyze and derive an efficient score
estimation using conditional score matching.

\subsubsection{Conditional Score Matching (CSM)}

First, the conditional density function of a node $X_{i}$ given its
parent nodes can be written as follows: 
\[
p_{X_{i}}(x_{i}|x_{\Pa_{i}})=p_{Z_{i}}(z_{i}|x_{\Pa_{i}})=p_{Z_{i}}(x_{i}-f_{i}(x_{\Pa_{i}}))
\]
where the noise $z_{i}=x_{i}-f_{i}(x_{\Pa_{i}})$ is assumed to have
a zero-mean density $p_{Z_{i}}$. Therefore, the density $p_{X_{i}}$
depends on both the observation $x_{i}$, the mean function $f_{i}$
and observation of the parent observations of the parent nodes $x_{\Pa_{i}}$.
Recursively, through $f_{i}$ and $x_{\Pa_{i}}$, $p_{X_{i}}$ depends
on all ancestor noise sources, denoted as $\mathbf{z}=\mathbf{z}_{\le i}$,
where $\le i$ indicates the indices of node $i$ and its ancestor
nodes. The score function of $x_{i}$ w.r.t. $z_{j}$, $j\le i$,
can be written as follows:

(i) If $j=i$ then: 
\begin{align}
s_{i}^{i}(x_{i}|\xb_{\Pa_{i}}) & =\nabla_{z_{i}}\log p_{X_{i}}(x_{i}|\xb_{\Pa_{i}})\nonumber \\
 & =\nabla_{z_{i}}\log p_{Z_{i}}(x_{i}-f_{i}(\xb_{\Pa_{i}}))\nonumber \\
 & =\frac{\partial\log p_{Z_{i}}(x_{i}-f_{i}(\xb_{\Pa_{i}}))}{\partial z_{i}}\nonumber \\
 & =\frac{\partial\log p_{Z_{i}}(z_{i})}{\partial z_{i}}\label{eq:chain-rule-1}
\end{align}

(ii) If $j\in\Pa_{i}$ then: 
\begin{align}
s_{j}^{i}(x_{i}|\xb_{\Pa_{i}}) & =\nabla_{z_{j}}\log p_{X_{i}}(x_{i}|\xb_{\Pa_{i}})\nonumber \\
 & =\nabla_{z_{j}}\log p_{Z_{i}}(x_{i}-f_{i}(\xb_{\Pa_{i}}))\nonumber \\
 & =\frac{\partial\log p_{Z_{i}}(x_{i}-f_{i}(\xb_{\Pa_{i}}))}{\partial z_{i}}\frac{(-\partial f_{i}(\xb_{\Pa_{i}})))}{\partial z_{j}}\nonumber \\
 & =\frac{\partial\log p_{Z_{i}}(z_{i})}{\partial z_{i}}\frac{(-\partial f_{i}(\xb_{\Pa_{i}}))}{\partial x_{j}}\nonumber \\
 & =s_{i}^{i}(x_{i}|\xb_{\Pa_{i}})\frac{(-\partial f_{i}(\xb_{\Pa_{i}}))}{\partial x_{j}}\label{eq:chain-rule-2}
\end{align}

(iii) If $j<i$ is \emph{not} a direct parent of $i$, then $j$ has
children $\Ch_{j}$ and the score recursively decomposes with respect
to each node $k\in\Ch_{j}$ as: 
\begin{align}
s_{j}^{i}(x_{i}|\xb_{<i}) & =\nabla_{z_{j}}\log p_{X_{i}}(x_{i}|\xb_{<i})\nonumber \\
 & =\sum_{k\in\Ch_{j}}\nabla_{z_{k}}\log p_{X_{i}}(x_{i}|\xb_{<i})\frac{\partial x_{k}}{\partial z_{j}}\nonumber \\
 & =\sum_{k\in\Ch_{j}}s_{k}^{i}(x_{i}|\xb_{<i})\frac{\partial f_{k}(\xb_{\Pa_{k}}))}{\partial z_{j}}\nonumber \\
 & =\sum_{k\in\Ch_{j}}s_{k}^{i}(x_{i}|\xb_{<i})\frac{\partial f_{k}(\xb_{\Pa_{k}}))}{\partial x_{j}}\label{eq:chain-rule-3}
\end{align}

As a result, the score function estimator only need to predict the
single score $\partial\log p_{Z_{i}}(z_{i})/\partial z_{i}$ for each
$z_{i}$. The score for the remaining noises $\mathbf{z}_{<i}$ can
then be computed using the chain rules (ii) and (iii) above. Therefore,
we will use two set of estimation models, (1) the score models $\{s^{i}(z_{i};\theta_{i})\}_{i=1,\dots,n}$
with parameters $\{\theta_{i})\}_{i=1,\dots,n}$; and (2) the mean
models $\{f_{i}(\mathbf{x}_{\Pa_{i}};\phi_{i})\}$ with parameters
$\{\phi_{i})\}_{i=1,\dots,n}$. We will describe how to train the
score models and the mean models next. Then, we will use an automatic
differentiation library to compute the partial derivatives $\partial f_{i}/\partial x_{j}$
for each mean model $f_{i}$ in the experiments. 

\subsubsection{Sampling Data and Gradient Paths}

To perform outlier attribution using our trained score networks (see Section~\ref{subsec:Score-based-Integrated-Gradient}), we first define the joint score function over all nodes $i=1, \dots, n$ as 
\[
\mathbf{s}(\mathbf{z}, t; \theta) = [s^1(z_1, t; \theta_1), \dots, s^n(z_n, t; \theta_n)]^T,
\]
where each $s^i$ operates on its local input $z_i$ and produces a gradient estimate. Given an outlier sample $\mathbf{z}$, we sample $m$ in-distribution trajectories $\{\mathbf{z}^{(k)}(t)\}_{k=1}^m$ by simulating the reverse-time SDE,
\begin{equation}\label{eq:reverse-sde}
d\mathbf{z} = -\sigma^2(t)\mathbf{s}(\mathbf{z}, t; \theta)\,dt + \sigma(t)\,d\mathbf{w},
\end{equation}
using Euler–Maruyama integration \citep{song2021scorebased}. These trajectories interpolate between the outlier and the distributional manifold.

Along each trajectory, we compute the gradient of the log-likelihood via integrated gradients, and apply the chain rules from Eqs.~\ref{eq:chain-rule-1}–\ref{eq:chain-rule-3} to estimate the influence of each latent variable $Z_j$ on the outlier $X_n$. This yields a Monte Carlo estimate of the attribution score in Eq.~\ref{eq:score-measure-MC}.

Algorithm~\ref{alg:score_sampling} provides the detailed procedure for sampling diffusion trajectories and storing the score vectors and diffusion steps. Unlike generative score-based models, we explicitly integrate along the reverse SDE to approximate a path integral over the log-density gradients, enabling principled attribution back to latent noise sources. This latent-space attribution approach distinguishes our method from prior work operating over observed variables.

\begin{algorithm}[H]
\caption{Sampling Trajectories and Storing Gradients and Diffusion Steps}
\label{alg:score_sampling}
\begin{algorithmic}[1]
\STATE \textbf{input:} trained score network $s_\theta(z, t)$, initial input sample $z_0$, time steps $t_1, \dots, t_n$, time step size $\Delta t$, maximum diffusion scale $\sigma_{\max}$.

\STATE \textbf{init:} initial sample for reverse diffusion $z \gets z_0$, set of scores $ \mathcal{S} \gets \emptyset$, set of diffusion steps $\mathcal{D} \gets \emptyset$.
\FOR{$i = 1$ to $n$}
    \STATE $\sigma^2 \gets (\sigma_{\max}^{2t_i} - 1) / (2 \log \sigma_{\max})$ \hfill \COMMENT{Diffusion coefficient}
    \STATE $\nabla z \gets s_\theta(z, t_i)$ \hfill \COMMENT{Evaluate score at $z$}
    \STATE $\mu_z \gets z + \sigma^2 \cdot \nabla z \cdot \Delta t$ \hfill \COMMENT{Euler–Maruyama mean update}
    \STATE $z_{\text{new}} \gets \mu_z + \sigma \cdot \sqrt{\Delta t} \cdot \epsilon$, where $\epsilon \sim \mathcal{N}(0, I)$ \hfill \COMMENT{Add Gaussian noise}
    \STATE $\mathcal{S} \gets \mathcal{S} \cup \{ \nabla z \}$ \hfill \COMMENT{Store score vector}
    \STATE $dz_i \gets z_{\text{new}} - z$ \hfill \COMMENT{Compute diffusion step}
    \STATE $\mathcal{D} \gets \mathcal{D} \cup \{ dz_i \}$ \hfill \COMMENT{Store diffusion step}
    \STATE $z \gets z_{\text{new}}$ \hfill \COMMENT{Set new $z$}
\ENDFOR
\RETURN score trajectories $\mathcal{S}=\{ \nabla s(z, t_i)\}_i$ of shape $(d, n)$ and diffusion steps $\mathcal{D}=\{dz_i\}_i$ of shape $(d, n)$.
\end{algorithmic}
\end{algorithm}


\subsubsection{Fitting Mean Models}

While the score-based framework in Section~3.2.3 enables the attribution of an outlier observation $z$ to the underlying latent perturbations via gradient paths, it does not by itself quantify how these perturbations propagate through the structural causal model. To estimate such propagation effects, particularly how noise sources $Z_j$ influence downstream variables $X_i$, we must recover the functional relationships $f_i$ that define the structural assignments $X_i = f_i(\text{Pa}_i) + Z_i$. These impact functions $f_i$ provide a functional semantics to the directed edges in the causal graph. Thus, we fit the regression models for each node $i$ as described below.

For each node $i$, we use a regression neural networks with parameters
$\phi_{i}$ for modeling the mean $\E X_{i}=\E(f_{i}(X_{\Pa_{i}};\phi_{i})+Z_{i})=f_{i}(X_{\Pa_{i}};\phi_{i})$,
where we assume zero mean noise $Z_{i}$. We train the network using
regularized regression:

\[
\phi_{i}^{*}=\argmin_{\phi_{i}}\E_{\mathbf{x}}\|\mathbf{x}_{i}-f_{i}(\mathbf{x}_{\Pa_{i}})\|^{2}+\eta\|\phi_{i}\|^{2}
\]
where $\eta$ is regularization parameter.

In the case of additive noise model, the additive noise model fit reduces to MSE error as mentioned above. For the location scale model,
\[
p(y \mid x) = \mathcal{N}(y \mid \mu_\phi(x), \sigma_\phi(x)^2)
\]
we use the following negative log-likelihood objective:
\[
\text{NLL}(\phi \mid x, y) = \frac{(y - \mu_\phi(x))^2}{2\sigma_\phi(x)^2} + \frac{1}{2} \log(\sigma_\phi(x)^2)
\]

\( f_j \) can be modeled using neural networks trained with mean squared error (MSE) loss 
\( \mathbb{E}[(X_j - f_j(\Pa_j))^2] \). The optimal predictor under this loss is the conditional expectation 
$ f_j(\Pa_j) = \mathbb{E}[X_j \mid \Pa_j], $
\cite[Sec.~1.5.5]{bishop2006pattern}. This result holds regardless of the noise distribution. The variance \( \mathrm{Var}(X_j \mid \Pa_j) \), represents the irreducible uncertainty in the target, captured by the noise term \( Z_j \).

\subsection{Discussions}

Extending the chain rule decomposition from additive noise models (ANMs), our key contribution is integrating this principle into a diffusion-based attribution framework operating in the latent noise space. By tracing the influence of noise perturbations through diffusion trajectories governed by the reverse SDE in Eq.~\ref{eq:reverse-sde}, we enable gradient-based attribution. Unlike prior approaches that compute attribution scores with respect to observed variables, we formulate attribution in terms of latent noise variables using score-based diffusion modeling and Monte Carlo sampling of gradient paths, as detailed in Section~\ref{sec:Methods}.

Our score-based attribution method inherits desirable properties from the Shapley value framework in cooperative game theory. Specifically, the attribution scores computed via integrated gradients over the negative log-likelihood (NLL) satisfy three of the four classic Shapley axioms—dummy, efficiency, and linearity. The fourth axiom, symmetry, does not generally hold in our setting; instead, we replace it with a context-specific asymmetry property induced by the underlying causal structure, wherein different variables may play inherently asymmetric roles in generating outliers. This theoretical grounding supports the faithfulness of our explanations: outlier nodes receive high attribution scores due to their disproportionate contributions to log-likelihood deviations. Furthermore, we interpret these scores as being approximately proportional to differences in tail probabilities, which we estimate using a triangle rule over the diffusion gradient path (see Fig.~\ref{fig:contribution-Shapley-vs-IG}).

Following~\cite{budhathoki2022causal}, we adopt the mechanism perturbation view of outliers, interpreting them as structural deviations in functional assignments or shifts in the noise-generating process. In the FCM framework, the exogenous noise variable determines the selection of deterministic mechanisms. Therefore, an outlier does not merely correspond to an unusual noise realization but reflects a structural shift in the response function. This view allows us to model outliers as arising from changes in mechanism, not just from stochastic noise. Rather than relying on density-based tail estimates (e.g., Eqs.~\eqref{eq:IT-score} and~\eqref{eq:residual}), which are often intractable or restricted due to privacy or storage constraints, we use Eq.~(11) in conjunction with integrated gradients to approximate the tail probability and explain the outlier behavior effectively.

Interpreting outliers as violations of causal invariance rather than extremal additive noise realizations is a more principled and informative perspective. Our approach is indeed aligned with this interpretation, as in \cite{budhathoki2022causal} and in \cite{nguyen2024root}. We further extend this idea by incorporating location-scale models, allowing the mean and variance of the noise to depend on the parents, thus supporting heteroscedastic and context-sensitive mechanism changes.

Our framework builds upon standard assumptions in causal inference~\cite{budhathoki2022causal,peters2017elements}, notably that the data-generating process follows a Functional Causal Model (FCM) structured as a directed acyclic graph (DAG) with causal sufficiency. We assume that the noise variables $\{Z_j\}$ are jointly independent and that the causal mechanisms are locally identifiable, meaning the score functions of the noise distributions and their perturbations are estimable from data, enabling gradient-based inference.

\section{Experiments }\label{sec:Experiments}

In this section, we evaluate our method, SIREN, in outlier attribution
tasks on random graph datasets and in a micro cloud services scenario.
We compare \textbf{SIREN} to the following baselines: 
\begin{enumerate}
\item \textbf{Naive}: This method simply uses the $z\text{-score}$ of the
marginal distribution of the observational data of each node $X_{i}$
as its contribution value to the target outlier measure. The ranking
of root causes is based on these $z\text{-scores}$, a higher ranking
corresponds to a higher z-score. 
\item \textbf{Traversal}: \cite{lin2018microscope,liu2021microhecl} assumes
a root cause is a node whose parent observations are normal, but there
is a descendant path of anomalous nodes linking it to the target outlier
node. 
\item \textbf{CIRCA}: \cite{li2022causal} models the functional causal
mechanisms (FCMs) using linear regression with Gaussian noise. The
authors compute the residual, Eq.\ \ref{eq:residual}, of an outlier
observation and use it for ranking root causes. 
\item \textbf{CausalRCA}: \cite{budhathoki2022causal} fits the FCMs using
linear regression. The authors use the Shapley value attribution method
to compute the attribution score of each node. They rewrite the target
outlier node as a function of all ancestor noises, including itself,
Eq.\ \ref{eq:noise-dependent-func}. The marginal contribution of
each noise source to the target outlier leaf, given all randomized
noise contexts is used for ranking root causes. 
\item \textbf{BIGEN}: \cite{nguyen2024root} fits a Bayesian linear regression
to the FCMs and use integrated gradients of the noise sources as the
leaf outlier contributions for ranking root causes. 
\end{enumerate}

\subsection{Random Graph Datasets}

Hyper-parameters for SIREN in Table.~\ref{tab:siren_networks}

\begin{table}[H]
\centering
\begin{tabular}{lccc}
\toprule
\textbf{Network} & \textbf{Number of layers} & \textbf{Hidden size} & \textbf{Activation} \\
\midrule
mean model $f_j$     & 2 & 100 & tanh \\
location model $\mu_j$ & 2 & 100 & tanh \\
scale model $\sigma_j$  & 2 & 100 & $0.5 + \tanh$ \\
score model $s_j$     & 3 & 100 & swish \\
\bottomrule
\end{tabular}
\caption{Network architecture details for SIREN.}
\label{tab:siren_networks}
\end{table}

BIGEN, CausalRCA, CIRCA, and Traversal use a similar model as the mean model $f_j$. The BIGEN model uses Gaussian Dropout to multiply each activation by noise sampled from a Gaussian distribution with mean 1.0 and standard deviation 0.1. All training is done using the Adam optimizer with default parameters in 100 epochs. 

In this experiment, we use random graph datasets to compare the performance
of all methods. We randomly generate 100 causal graphs with varying
number of nodes in the range from 50 to 100. For each graph, we select
the subgraph with a depth of at least 10 and having only one leaf
for study. For the FCMs, we use 3-layer MLPs with 50 hidden units
and a ReLU activation function in the hidden layer. We then draw normal
data from these FCMs, following the causal structure. To synthesize
the outliers, we randomly select between 1 to 3 nodes as root causes.
We inject outlier noise into the nodes to create the ground truths
by changing the scale of these noise sources three times. The noise
distribution for each FCM is randomly selected as a mixture of Gaussians
with 2 components, with means sampled from $\Ncal(0,1)$, both having
the same variance of $1$, and the prior weights are $[\alpha,1-\alpha]$,
where $\alpha\in[0.1,0.9]$.

Table.\ \ref{tab:random-graph} compares the average NGCG@$k$ ranking
on random graph datasets. SIREN outperforms all the baselines in root
cause detection ranking on average across top-$k$ measures. CIRCA
and BIGEN perform comparably and are better than CausalRCA, Traversal,
and Naive method.

\begin{table}
\begin{centering}
\begin{tabular}{|c|c|}
\hline 
 & Average NGCG@$k$ ranking\tabularnewline
\hline 
\hline 
SIREN (ours) & \textbf{$\mathbf{85.0\pm11.5}$}\tabularnewline
\hline 
CIRCA & $73.1\pm18.1$\tabularnewline
\hline 
BIGEN & $75.7\pm17.5$\tabularnewline
\hline 
CausalRCA & $58.6\pm18.3$\tabularnewline
\hline 
Traversal & $51.4\pm21.8$\tabularnewline
\hline 
Naive & $41.6\pm19.6$\tabularnewline
\hline 
\end{tabular}
\par\end{centering}
\caption{Average NGCG@$k$ ranking in the root cause attribution task on random
graph datasets.}
\label{tab:random-graph}
\end{table}

Fig.\ \ref{fig:random-graph} shows the detailed top-$k$ rankings
of the methods. At the top-$1$ ranking, all three methods - SIREN,
CIRCA, and BIGEN - correctly detect the first root cause, performing
closely at about 83\%. However, for the top-$2$ and top-$3$ rankings,
only SIREN maintains accurate rankings, while CIRCA and BIGEN drop
to about 60\% NGCG score. This suggests that the data distribution
is complex and requires better distribution matching to detect more
than one root cause in ancestor nodes. SIREN's score matching ability
gives it an advantage in this case. At higher top-k rankings, all
methods seem to converge, but SIREN still performs the best among
the six methods.
\begin{figure}
\begin{centering}
\includegraphics[width=0.95\columnwidth]{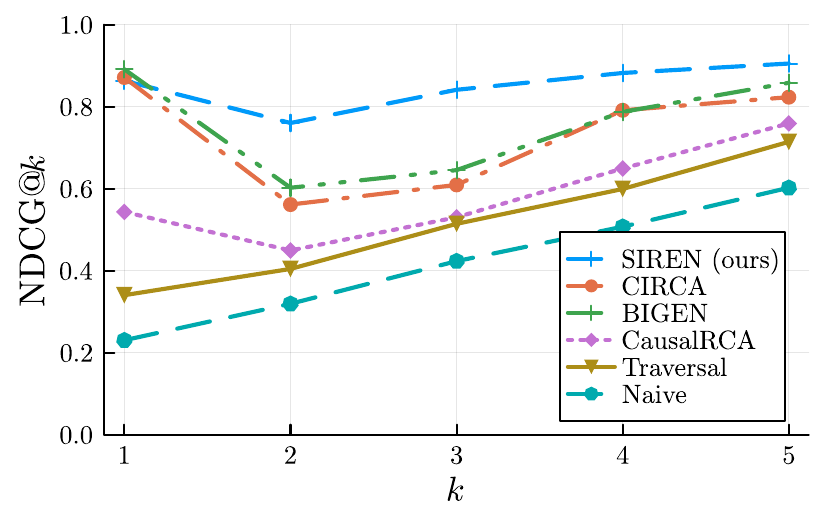}
\par\end{centering}
\caption{NDCG@$k$ ranking across different $k$ values.}
\label{fig:random-graph}
\end{figure}

\subsection{Root Causes of Cloud Service Latencies}

In this experiment, we compare our methods with the baselines in a
simpler, realistic scenario within a small microservice environment.
Specifically, we investigate and identify the root causes of observed
latencies in cloud services for an online shop \cite{guo2022joint,ikram2022root,dowhy_gcm}.
We consider normal random delay noises at each service $X_{j}$ to
be sampled from a half-normal distribution with a mean between 100--500
ms and a scale between 100--200 ms. We randomly select between 1--3
nodes and change their means and noise scales three times. Our objective
is to identify the root causes of unwanted latency experienced by
customers when processing online orders. This service interacts with
multiple other web services, which are represented by a causal dependency
graph that includes ten additional services \cite{dowhy_gcm}. We
assume that we are observing latencies in the order confirmation at
the website's leaf node, and that all services operate in synchronization.
We use similar models as in the previous section for this experiment,
assuming the causal graph is given, and fit the noisy FCM models to
the training data collected during normal operations. Due to the linearity
of the problem, we use linear regression instead of neural network
for fitting the mean functions in SIREN.

Table~\ref{tab:micro} compares the root cause rankings of all methods.
It shows that SIREN, BIGEN, and CausalRCA can correctly rank root
causes of delays on average for these online shop cloud services.
CIRCA performs reasonably, achieving 72.3\% NDCG@$k$ on average,
which may be due to the difficulty posed by asymmetric noise sources,
making it fail in this context because it simply uses the $z$-score
with the fitted mean function. This creates additional challenges
for the Traversal and Naive methods.

\begin{table}
\begin{centering}
\begin{tabular}{|c|c|}
\hline 
 & Average NGCG@$k$ ranking\tabularnewline
\hline 
\hline 
SIREN (ours) & $\mathbf{92.1\pm3.3}$\tabularnewline
\hline 
CIRCA & $72.3\pm8.1$\tabularnewline
\hline 
BIGEN & $89.4\pm3.7$\tabularnewline
\hline 
CausalRCA & \textbf{$\mathbf{92.2\pm2.9}$}\tabularnewline
\hline 
Traversal & $50.4\pm3.6$\tabularnewline
\hline 
Naive & $43.8\pm12.2$\tabularnewline
\hline 
\end{tabular}
\par\end{centering}
\caption{Average NDCG@$k$ ranking in an online shop cloud service.}
\label{tab:micro}
\end{table}

Fig.~\ref{fig:micro-service} shows that CausalRCA can perfectly
rank the first root cause, while it is comparable to SIREN and BIGEN
for larger $k$. All methods exhibit better NDCG@$k$ scores with
larger $k$ values, as the chances of relevant root causes appearing
increase. 
\begin{figure}
\begin{centering}
\includegraphics[width=0.95\columnwidth]{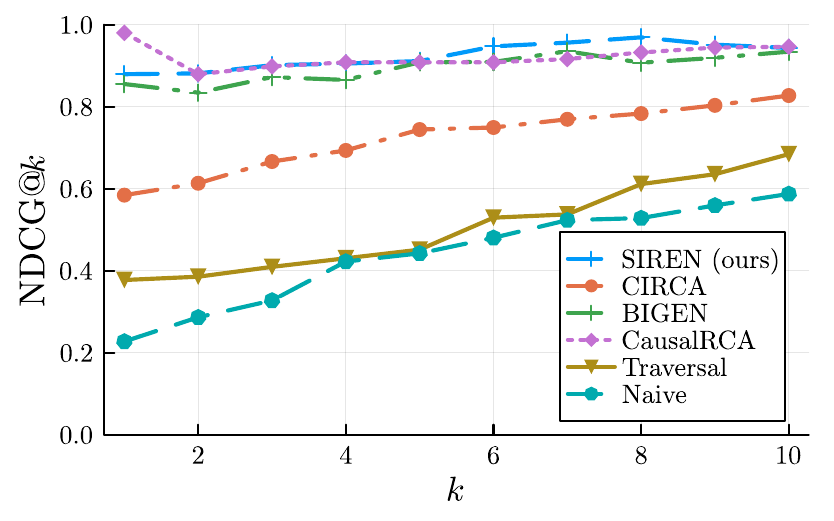}
\par\end{centering}
\caption{NDCG@$k$ with varying the number of $k$.}
\label{fig:micro-service}
\end{figure}

\subsection{Root Causes of Supply Chain Latencies}

In this experiment, we simulated a supply chain latency dataset with the interactions between vendors and retailers. The real-world delays and decision-making overheads introduce stochastic variability. This experiment captures realistic structural dependencies, time-lagged effects, and non-additive noise patterns that arise in business operations and logistics, closely mirroring real causal phenomena. The results showed that our model effectively identifies latent root causes of outliers observed at the final "Received" node, and outperforms baseline methods on NDCG@k metrics (Table 2), showing robustness and accuracy in attribution under realistic noise and delay distributions.

We based on the root cause analysis in a real world supply chain scenario using the simulator in [Nguyen et al., 2024]. This experiment addresses the challenge of identifying root causes of outlier events (e.g. delays) in a supply chain system. The setting simulates a realistic retail-vendor chain with stochastic delays, modeled using location scale data generating process. The supply chain works as follow. A retailer places orders based on forecasted demand and inventory constraints. A vendor then confirms orders and ships goods with random delays. The delays simulate business overheads like managerial decisions or variable operations. Outliers are injected by perturbing the noise, e.g. Uniform(3, 5). The objective is to identify the root causes of outliers at the final node (Received) by ranking influential upstream variables. Table~\ref{tab:ndcg_k} shows the average NDCG@k ranking results from different runs:

\begin{table}[H]
\centering
\caption{NDCG@k ranking across five k values by different methods}
\label{tab:ndcg_k}
\begin{tabular}{lcccc}
\toprule
\textbf{Method} & \textbf{k=1} & \textbf{k=2} & \textbf{k=3} & \textbf{k=4} \\
\midrule
SIREN & $92.7 \pm 2.1$  & $81.6 \pm 6.7$  & $80.4 \pm 7.0$  & $82.0 \pm 5.2$  \\
CIRCA      & $71.0 \pm 41.6$ & $72.0 \pm 27.3$ & $73.7 \pm 18.5$ & $80.5 \pm 15.7$ \\
BIGEN      & $56.8 \pm 43.2$ & $56.9 \pm 36.3$ & $63.6 \pm 29.2$ & $72.5 \pm 21.5$ \\
CausalRCA  & $77.0 \pm 22.0$ & $68.1 \pm 16.1$ & $65.5 \pm 12.5$ & $71.6 \pm 12.6$ \\
Traversal  & $62.8 \pm 31.5$ & $70.3 \pm 25.4$ & $73.7 \pm 20.3$ & $77.7 \pm 14.3$ \\
Naive      & $53.5 \pm 29.5$ & $61.1 \pm 18.9$ & $64.4 \pm 13.9$ & $69.8 \pm 7.7$  \\
\bottomrule
\end{tabular}
\end{table}

\section{Related Work}

A closely related method to ours is BIGEN \cite{nguyen2024root},
in which the authors also use integrated gradients for root cause
attribution. However, their method only works for linear Gaussian
FCMs, which limits root cause analyses to linear and Gaussian noise
assumptions. They also require access to the training data for attribution.
Another closely related method is the pioneering CausalRCA \cite{budhathoki2022causal},
where the authors base their approach on the causal structure of the
data and use Shapley values for outlier attribution. However, this
method requires linear FCMs, and training data for root cause attribution,
and its computation is expensive. CIRCA \cite{li2022causal} models
the data using linear FCMs and uses the distance to the predicted
mean for ranking root causes. \cite{okati2024root} utilize the topological
order of a causal structure and consider only a single root cause.
The authors proposed to traverse the nodes in reverse topological
order and use changes in log probabilities to identify the root cause.
Earlier methods employing the traversal approach include \cite{chen2014causeinfer,lin2018microscope,liu2021microhecl}.
However, these methods use a heuristic algorithm to identify a root
cause by checking its parents for non-anomalous activity and finding
a path of anomalous nodes linking to the target outlier.

\section{Conclusion}

We proposed SIREN, a score-based integrated gradient framework for root cause attribution in complex, noisy systems. SIREN attributes an outlier measure at a leaf node to its ancestor nodes via integrated gradients, identifying likely root causes. Our framework is built upon a conditional score matching model that enables efficient computation of integrated gradient-based outlier scores, facilitated by a decomposed score function. Unlike prior methods, SIREN supports nonlinear, heteroscedastic causal models by modeling location-scale noise and interpreting outliers as violations of causal invariance rather than extreme noise realizations. Extensive theoretical analysis and empirical evaluations confirm the effectiveness, scalability, and interpretability of our approach.

\bibliographystyle{unsrt}
\bibliography{causal}

\end{document}